\documentclass[nohyperref]{article}

\usepackage{microtype}
\usepackage{graphicx}
\usepackage{subfigure}
\usepackage{booktabs} 

\usepackage{hyperref}

\usepackage[accepted]{icml2022}

\usepackage{amsmath}
\usepackage{amssymb}
\usepackage{mathtools}
\usepackage{amsthm}

\usepackage[capitalize,noabbrev]{cleveref}

\usepackage{graphicx}
\usepackage{lipsum}
\DeclareMathOperator{\argsort}{argsort}

\usepackage{multirow} 
\usepackage{soul}
\usepackage{changepage,threeparttable} 
\usepackage{tabularx}
\usepackage{float}

\newcolumntype{Y}{>{\centering\arraybackslash}X}

\theoremstyle{plain}
\newtheorem{theorem}{Theorem}[section]

\newtheorem{lemma}[theorem]{Lemma}

\theoremstyle{definition}

\theoremstyle{remark}

\icmltitlerunning{Model Stealing Defenses with Gradient Redirection}

\begin{document}

\twocolumn[
\icmltitle{How to Steer Your Adversary: Targeted and Efficient\\Model Stealing Defenses with Gradient Redirection}




\begin{icmlauthorlist}
\icmlauthor{Mantas Mazeika}{yyy}
\icmlauthor{Bo Li}{yyy}
\icmlauthor{David Forsyth}{yyy}
\end{icmlauthorlist}

\icmlaffiliation{yyy}{UIUC}

\icmlcorrespondingauthor{Mantas Mazeika}{mantas3@illinois.edu}

\icmlkeywords{Machine Learning, ICML}

\vskip 0.3in
]



\printAffiliationsAndNotice{}  

\begin{abstract}
Model stealing attacks present a dilemma for public machine learning APIs. To protect financial investments, companies may be forced to withhold important information about their models that could facilitate theft, including uncertainty estimates and prediction explanations. This compromise is harmful not only to users but also to external transparency. Model stealing defenses seek to resolve this dilemma by making models harder to steal while preserving utility for benign users. However, existing defenses have poor performance in practice, either requiring enormous computational overheads or severe utility trade-offs. To meet these challenges, we present a new approach to model stealing defenses called gradient redirection. At the core of our approach is a provably optimal, efficient algorithm for steering an adversary's training updates in a targeted manner. Combined with improvements to surrogate networks and a novel coordinated defense strategy, our gradient redirection defense, called $\text{GRAD}{}^2$, achieves small utility trade-offs and low computational overhead, outperforming the best prior defenses. Moreover, we demonstrate how gradient redirection enables reprogramming the adversary with arbitrary behavior, which we hope will foster work on new avenues of defense.
\end{abstract}

\begin{figure}[t]
\begin{center}
\includegraphics[width=0.46\textwidth]{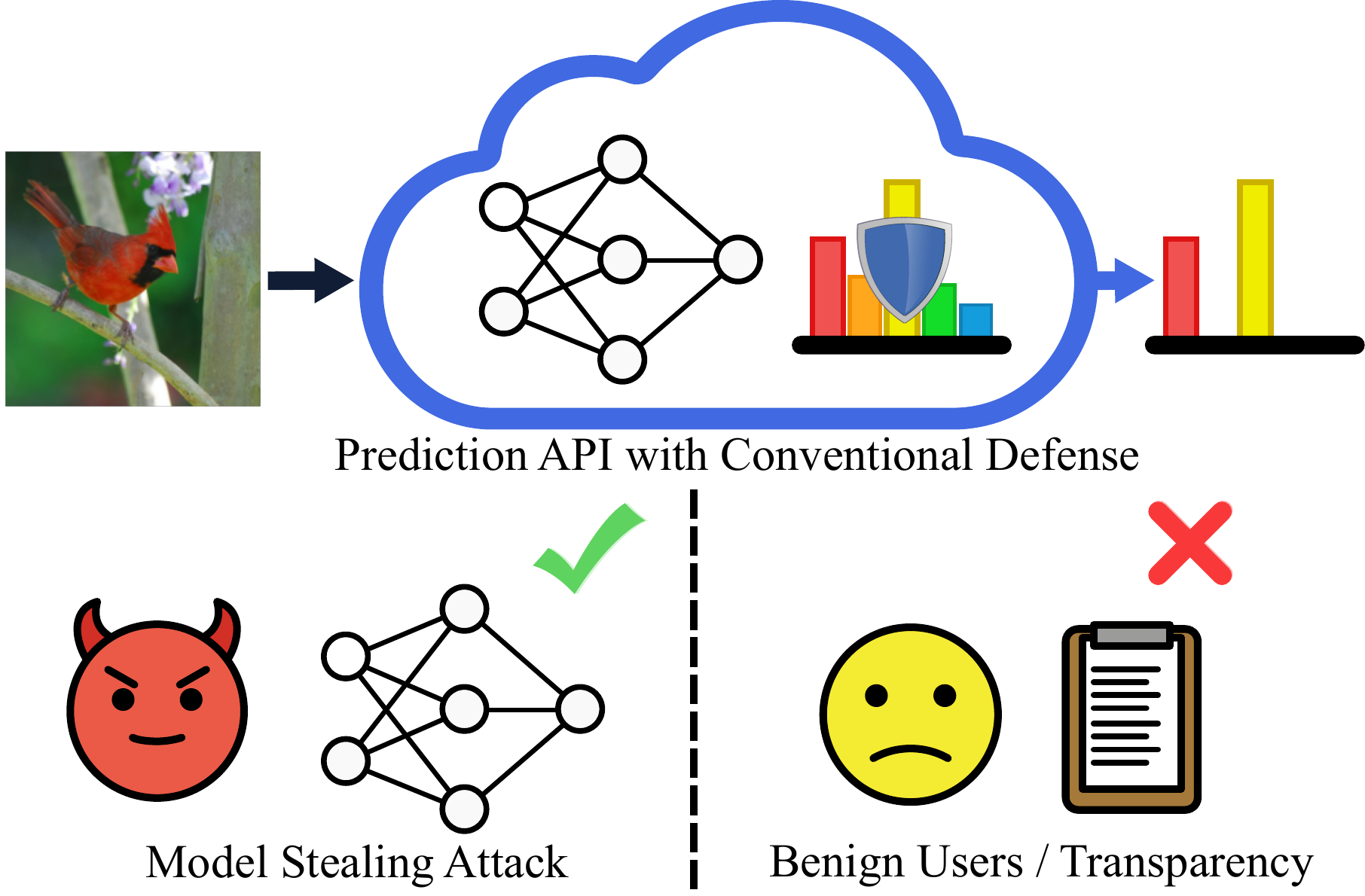}
\end{center}
\vspace{0pt}
\caption{
High-profile prediction APIs such as the OpenAI API and AI21 Studio truncate posteriors. This serves as a rudimentary defense against model stealing attacks, protecting financial investments. However, it also has significant trade-offs. Namely, it harms benign users who might otherwise benefit from the withheld predictions, and it reduces transparency.
}\label{fig:motivation_figure}
\vspace{-10pt}
\end{figure}

\section{Introduction}

As deep neural networks become more capable and economically valuable, responsibly democratizing access to the best available models could lead to widespread social good. Owners may not wish to just publish parameters, as this opens the door to malicious use and provides no direct return on investment \cite{radford2019better, buchanan2021truth}. Prediction APIs have emerged as a solution to these problems, as they enable filtering out harmful use cases and support a software-as-a-service business model. Moreover, they allow users with minimal computational resources to access the largest available models at reasonable costs.

The promise of prediction APIs is hampered by the fact that they are vulnerable to model stealing attacks. Malicious users can clone the functionality of an API by gathering a dataset of queries and responses and training their own model. This allows them to circumvent the expensive process of manual data curation, which can be a multi-million dollar investment for API providers \cite{hendrycks2021cuad, tramer2016stealing}. Additionally, model theft can be used as a stepping stone for mounting evasion attacks, which could result in critical failures in downstream applications using the API \cite{papernot2017practical}. In this paper, we investigate methods for defending against model stealing attacks.

In practice, prediction APIs often employ rudimentary obfuscation measures, such as truncating predicted posteriors to a minuscule fraction of their original size. For instance, the OpenAI API and AI21 Studio truncate to $2\%$ and $0.03\%$ of the available information, respectively. While this can protect against model stealing, it has major side effects. As illustrated in Figure \ref{fig:motivation_figure}, truncating posteriors reduces external transparency and harms benign users who might otherwise benefit from the withheld predictions.

Recent developments in model stealing defenses provide an avenue towards enabling API providers to share more information about their models without fear of extraction attacks. The objective of model stealing defenses is to make models harder to steal without substantially altering posterior predictions. In recent years, several works have investigated perturbation-based defenses that seek to maximize adversary error while minimally altering posteriors. However, these defenses have poor performance in practice, as they either incur infeasible computational overheads or require large utility trade-offs to be effective.

To overcome these challenges, we propose a new approach to perturbation-based model stealing defenses, which we call gradient redirection. At the core of our approach is an efficient, provably optimal algorithm for steering an adversary's training updates. Unlike prior methods, gradient redirection enables altering the trajectory of an extraction attack in a targeted manner, enabling a wide range of possible defenses. Using gradient redirection, we develop an efficient defense called $\text{GRAD}^2$ that incorporates improved surrogate networks and a novel coordinated defense strategy. In extensive experiments, we find that $\text{GRAD}^2$ outperforms prior defenses across multiple threat models. Moreover, we show how gradient redirection enables reprogramming the adversary with arbitrary behavior, including hidden watermarks. Experiment code is available at \url{https://github.com/mmazeika/model-stealing-defenses}.


\section{Related Work}
\paragraph{Model Stealing Attacks.}
Numerous works have explored the vulnerability of prediction APIs to model extraction attacks, where the adversary's goal is to obtain a copycat network with similar functionality to the prediction API. In early work, \citet{tramer2016stealing} identify a number of threat models and show that extraction of simple model classes is possible. For deep neural networks, \citet{papernot2017practical} show that model extraction is possible and can facilitate subsequent evasion attacks. Both these works assume the ability to adaptively probe the API with optimized queries, which can be accurately detected in some cases by monitoring query patterns \citep{8806737, pal2020activethief}. Thus, we focus our investigation on the case where adversary queries have no carefully crafted inputs or sequential structure. By leveraging knowledge distillation \citep{hinton2015distilling}, several works have shown that it is possible to steal the functionality of deep neural networks using weakly related queries \citep{orekondy2019knockoff} or even random queries \citep{krishna2019thieves}, although attacks with queries more related to the target task typically have stronger performance.\looseness=-1

\begin{figure}[t]
\begin{center}
\includegraphics[width=0.46\textwidth]{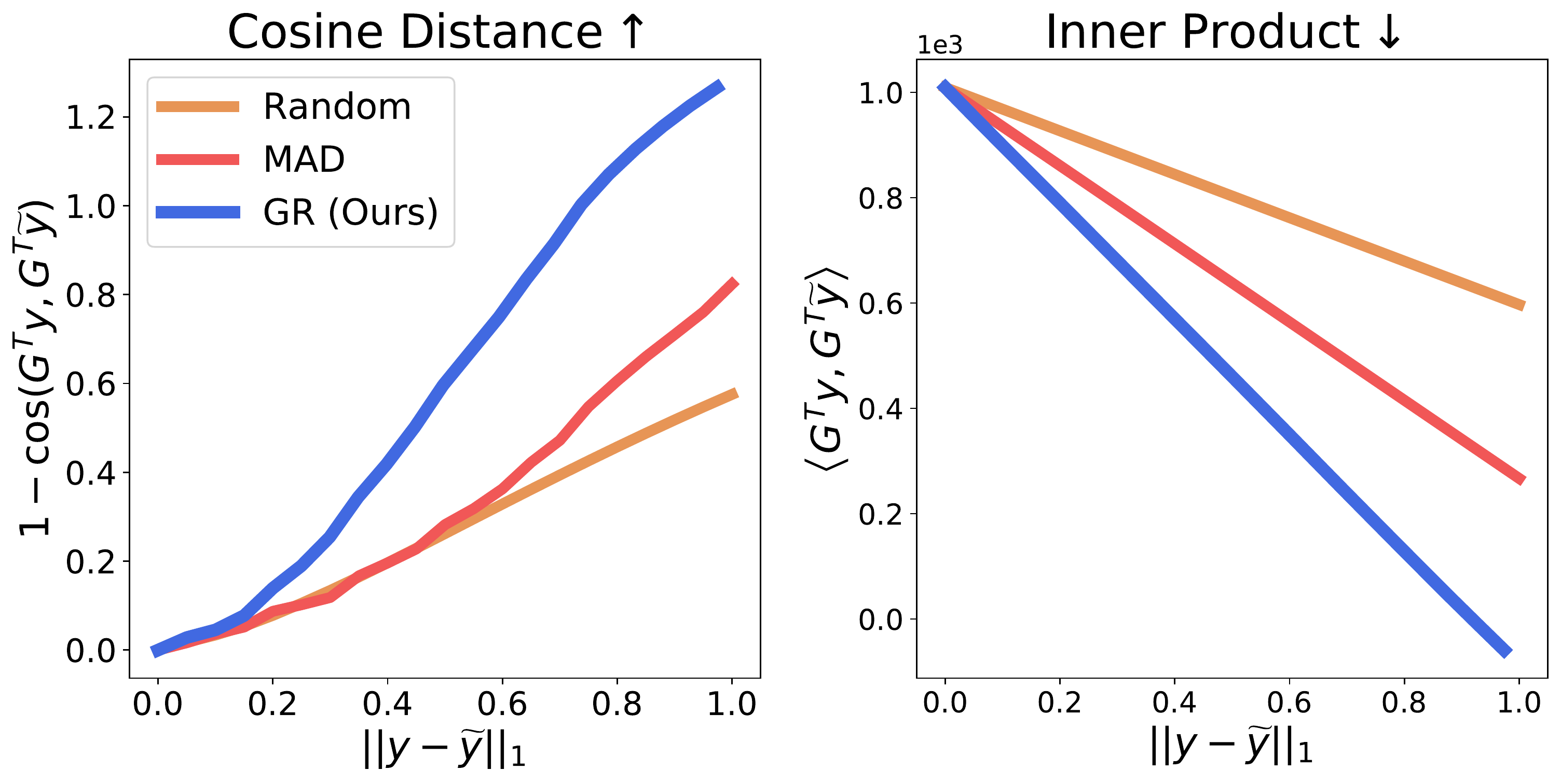}
\end{center}
\vspace{-10pt}
\caption{Compared to the previous state-of-the-art (MAD), gradient redirection (GR) results in larger perturbations to update gradients for a model stealing attack. MAD maximizes angular deviations with a heuristic algorithm (left), and GR minimizes inner product with a provably optimal algorithm (right). For both metrics, gradient redirection gives substantially more leverage.}\label{fig:surrogate_metrics}
\vspace{-10pt}
\end{figure}

\begin{figure*}[t]
\begin{center}
\includegraphics[width=\textwidth]{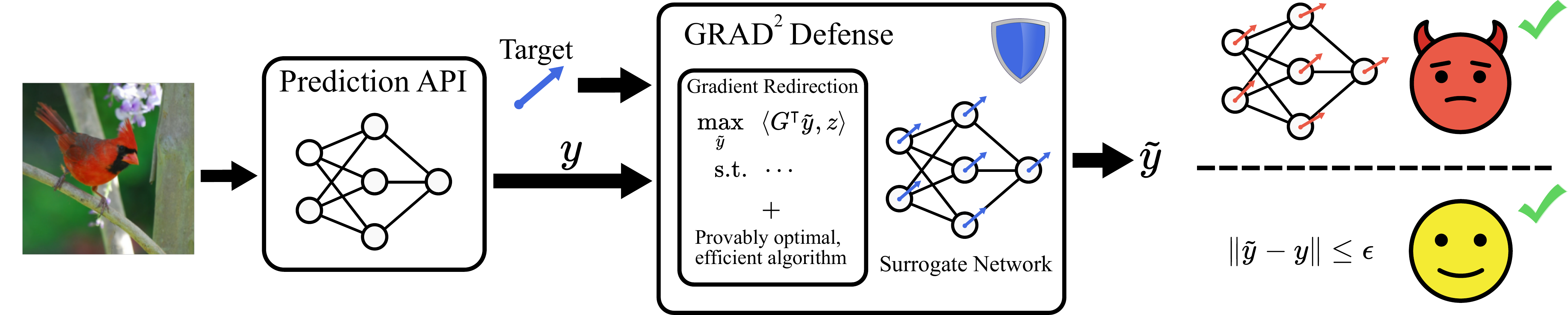}
\end{center}
\vspace{0pt}
\caption{
A user submits an image {\bf left} to a prediction API, producing posteriors $y$. Our defense outputs an adjusted posterior $\widetilde{y}$ to prevent a malicious user from stealing the model ({\bf red, right}). The adjustment is chosen by an efficient, provably optimal, algorithm to produce the largest error in a surrogate network's gradient, which transfers to the unknown adversary network. Extensive experiments show that the stolen model has a significant loss in accuracy. Moreover, the benign user is happy, because the adjustments are guaranteed to be small.\looseness=-1
}\label{fig:method_figure}
\vspace{-10pt}
\end{figure*}

\paragraph{Model Stealing Defenses.}
Defenses against model stealing fall into two complimentary approaches, which can be combined in a swiss cheese model to obtain robust protection \cite{hendrycks2021unsolved}. The first line of defense attempts to directly thwart model extraction attempts by returning modified or censored predictions from the API without significantly reducing performance for benign users. The simplest such defense is truncating posteriors or only outputting the predicted class \cite{tramer2016stealing}. This defense is used in real prediction APIs, such as the OpenAI API and AI21 Studio where posterior probabilities are truncated to the top $2\%$ and $0.03\%$ of values, respectively. Unfortunately, this reduces utility to benign users, precluding certain applications, and is harmful to external transparency.

Rather than truncate posteriors, several recent works have investigated slightly perturbing posteriors to derail extraction attempts. \citet{8844598} introduce ambiguity into clean posteriors, which lowers the accuracy of the stolen model while preserving the defender's accuracy. Building on this intuition, \citet{kariyappa2020defending} train a misinformation network to predict incorrect posteriors. Similarly, \citet{krishna2019thieves} replace clean posteriors with harmful posteriors based on how anomalous the query is \cite{hendrycks2018deep}. Note that this defense strategy assumes a setting where attackers have difficulty obtaining queries close to the defender's training distribution, which may not be true in practice.

The prediction poisoning approach introduced by \citet{Orekondy2020PredictionPT} demonstrates that posteriors can be optimally perturbed to derail the backward pass of an extraction attack, similar to adversarial examples for the forward pass \cite{szegedy2013intriguing}. They choose perturbations to poison the adversary's update gradient by maximizing angular deviation with the clean gradient. This results in a strong defense, but it involves computing the full Jacobian of the network's posterior, an immense overhead cost that scales with the number of labels. For some applications, such as language modeling, the number of labels can be in the tens or hundreds of thousands. This is noted by \citet{wallace2020imitation}, who focus on machine translation and adopt a best-of-k sampling procedure for selecting a perturbed prediction to derail the adversary's backward pass. However, this workaround still incurs an overhead of $k$ samples and may lead to a suboptimal defense. Building on these works, we develop a new approach for optimally redirecting an adversary's gradient in a targeted and efficient manner.

\paragraph{Watermarking Defenses.}
A second overarching approach to model stealing defenses is digital watermarking. Many works purposefully insert backdoors into networks \cite{10.1145/3196494.3196550, adi2018turning}, which enable proof of ownership. In the setting of model stealing, adversaries do not directly access model parameters and may be able to evade the watermark backdoor queries. Thus, different watermarking strategies are required to enable identifying models stolen through prediction APIs. \citet{szyller2021dawn} extend watermarks to this setting by outputting incorrect predictions for a small percentage of queried inputs. This enables the defender to prove ownership by querying the stolen model on these watermark examples. However, a downside of this defense is that the adversary knows its own training set, and thus has access to the watermarked inputs. In our investigation, we find that gradient redirection with strong surrogates enables black-box reprogramming of the adversary, which can insert watermarks into stolen models that are completely unknown to the adversary.
\section{Threat Model}
We consider interactions between a single attacker and defender. The defender allows users to access a deep neural network $g$ through a prediction API, and the attacker attempts to steal the functionality of the defender's model through querying the API and training a clone model $f$ on the resulting input-output pairs. The attacker's goal is to obtain high accuracy on the defender's test set.

\subsection{Attacker's Strategy}
The attacker chooses a query $x$ and receives a prediction from the defender's network. To avoid detection defenses \cite{8806737, chen2020stateful}, we assume the attacker sends queries that mimic benign queries. That is, we avoid adaptive querying strategies with unusual temporal structures \cite{papernot2017practical, orekondy2019knockoff}. In particular, we consider the Knockoff Nets attack without adaptive querying \cite{orekondy2019knockoff}, which is similar to knowledge distillation \cite{hinton2015distilling} and yields state-of-the-art attack performance \cite{Orekondy2020PredictionPT, kariyappa2020defending}.

\noindent\textbf{Knockoff Nets Attack.}\quad
Let $g$ be the defender's network, and let $f$ be the attacker's clone model parametrized by $\theta \in \mathbb{R}^d$. The attacker chooses a dataset of queries $\mathcal{Q}$ beforehand and sends queries $x \in \mathcal{Q}$ in random order. Let $y = g(x)$ be the defender network's posterior on query $x$. The loss of the adversary's clone model $f$ on the example $(x, y)$ is $H(y, f(x)) = -\sum_i y_i \log f(x)_i$. The adversary's update gradient on this example is the negative gradient of the loss with respect to the parameters $\theta$. Note that this update gradient can be written as
\begin{align*}
    -\nabla_\theta H(y, f(x; \theta)) &= \sum_i y_i \nabla_\theta \log f(x; \theta)_i \\
    &= y^\intercal G,
\end{align*}
where $G = \nabla_\theta \log f(x; \theta)$ is the $n \times d$ Jacobian matrix of the log-posteriors of $f$. Through collecting large numbers of queries and using standard techniques for training deep neural networks, the adversary can mount a successful extraction attack.

\noindent\textbf{Query Distribution.}\quad
An important variable in the adversary's attack is the choice of query dataset $\mathcal{Q}$. Prior work uses query datasets with varying levels of similarity to the defender's training distribution, ranging from highly similar \cite{8844598, tramer2016stealing} to unrelated \cite{Orekondy2020PredictionPT, kariyappa2020defending}. In some cases, collecting many queries similar to the defender's training distribution may be easy, and in other cases it may be challenging. Hence, we consider both possibilities in our threat model. We refer to adversaries as \textit{distribution-aware} if they use queries sharing semantic content with the defender's training distribution. We refer to adversaries as \textit{knowledge-limited} if they use queries with disjoint semantic content from the defender's training distribution.

\subsection{Defender's Objective}
The defender has no access to the adversary's parameters. Indeed, if the adversary mimics benign queries then the defender may not even know they are under attack. Thus, the challenge for the defender is to mitigate model stealing attempts at all times while preserving the utility of their API for benign users. This is typically accomplished by returning modified posteriors that make model stealing hard are representative of the clean posteriors.

A primary measure of utility to benign users is the defender's classification error on the test set. However, it is also important to consider the overall modification to returned posteriors, as benign users can derive substantial value from guarantees that the perturbed posteriors are close to the clean posterior. That is, a perturbed posterior $\widetilde{y}$ should satisfy $\|\widetilde{y} - y\|_1 \leq \epsilon$, where $y$ is the clean posterior and $\epsilon$ is small. A good defense will reduce the accuracy of clone models while minimally increasing classification error on the test set and $\ell_1$ distance to the clean posteriors.

\section{Approach}\label{sec:gradient_redirection}




A recent insight in the development of model stealing defenses is the notion of adversarial perturbations to the backward pass \cite{Orekondy2020PredictionPT, wallace2020imitation}. That is, perturbations to the defender's posterior can be designed to maximally change the attacker's update gradient. Unfortunately, existing methods for computing these perturbations are inefficient, requiring sample-based approaches or hundreds of backward passes for a single API query. Additionally, they lack flexibility as model stealing defenses, since they only seek to maximize angular deviation with the original update gradient. Here, we describe a new approach that we call gradient redirection, which pushes the attacker's update gradient in a target direction and is efficient to compute.\looseness=-1

\subsection{Gradient Redirection Problem}
For simplicity, we start by assuming white-box access to the adversary's network $f$. We operationalize the defender's goal as minimally perturbing the defender's posterior $y$ in order to maximally push the adversary's update gradient in a target direction $z \in \mathbb{R}^d$. For a given distillation example $(x, y)$, we want to solve the optimization problem
\begin{align}\label{equation:gradient_redirection}
    \max_{\widetilde{y}} \quad & \langle G^\intercal \widetilde{y},  z\rangle \\
    \textrm{s.t.} \quad & \mathbf{1}^\intercal \widetilde{y} = 1\nonumber \\
    & \widetilde{y} \succeq 0\nonumber \\
    & \| \widetilde{y} - y \|_1 \leq \epsilon,\nonumber
\end{align}
where $y$, $G = \nabla_\theta \log f(x; \theta)$, $z \in \mathbb{R}^d$, and $0 \leq \epsilon < 2$ are fixed. Typically we also have $y \in \Delta^{n-1}$, although for proofs we may have $\sum_i y_i < 1$. Note that this is a linear program, so in theory we could find the optimal $\widetilde{y}$ with performant LP solvers such as affine scaling variants of Karmarkar's algorithm \cite{adler1989implementation}. However, in real world cases, $n$ may be in the tens of thousands, requiring upwards of 10GB just to store the constraint matrix for a single distillation example. As prediction APIs often handle high volumes of queries, solving this optimization problem with existing methods would be far too costly in practice.

\begin{figure}[t]
\vspace{-7pt}
\begin{algorithm}[H]
\caption{Gradient Redirection}
\label{algorithm:gradient_redirection_l1}
\begin{algorithmic}
   \STATE {\bfseries Input:} $G$, $z$, $y$, $\epsilon$
   \STATE {\bfseries Output:} $\widetilde{y}$
   \STATE $\widetilde{y} \leftarrow y$
   \STATE $s \leftarrow \argsort(Gz)$
   \STATE $\widetilde{y}_{s_n} \leftarrow \min\!\left( y_{s_n} + \epsilon / 2,\, 1 \right)$
   \STATE $\lambda \leftarrow 0$
   \STATE $t \leftarrow 1$
   \WHILE{$t < n$}
   \STATE $\widetilde{y}_{s_t} \leftarrow \max\!\left( y_{s_t} - (\epsilon / 2 - \lambda),\, 0 \right)$
   \IF{$y_{s_t} - (\epsilon / 2 - \lambda) > 0$}
   \STATE {\bfseries Return} $\widetilde{y}$
   \ENDIF
   \STATE $\lambda \leftarrow \lambda + y_{s_t}$
   \STATE $t \leftarrow t + 1$
   \ENDWHILE
\end{algorithmic}
\end{algorithm}
\vspace{-20pt}
\end{figure}



\begin{figure*}[t]
\vspace{-5pt}
\begin{center}
\includegraphics[width=\textwidth]{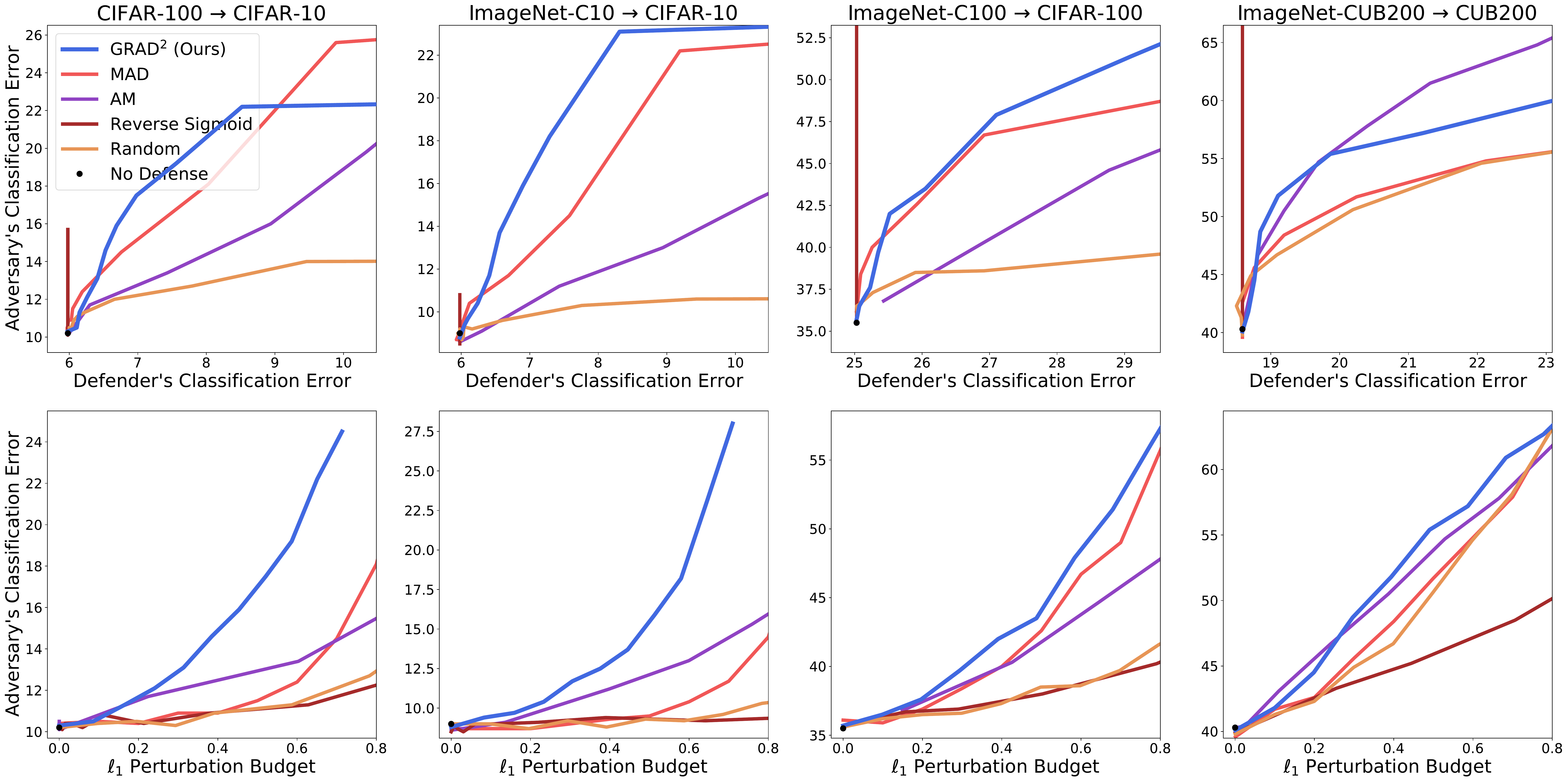}
\end{center}
\vspace{0pt}
\caption{
For most defense budgets, $\text{GRAD}^2$ induces higher classification error in the adversary than the best prior methods. While Reverse Sigmoid is efficient in terms of classification error, this performance requires an unreasonably large $\ell_1$ perturbation budget, which renders the defense unusable in practice. By contrast, $\text{GRAD}^2$ has strong performance on both budget metrics.
}\label{fig:plots1}
\vspace{-10pt}
\end{figure*}

\subsection{Gradient Redirection Algorithm}

We propose an efficient and provably optimal algorithm to solve the gradient redirection problem. First, we note that our problem is structurally similar to the fractional knapsack problem, which is solved in linearithmic time by a greedy algorithm \cite{dantzig201626}. Drawing from this connection, we propose a knapsack-like greedy algorithm for solving gradient redirection in Algorithm \ref{algorithm:gradient_redirection_l1}. The inner product objective (\ref{equation:gradient_redirection}) can be rewritten as $\widetilde{y}^\intercal G z$, where $Gz \in \mathbb{R}^n$ can be interpreted as a value vector.

Our algorithm initializes $\widetilde{y}$ as $y$. It then proceeds by taking probability mass from indices of $Gz$ with low values and putting as much mass as possible in the index of $Gz$ with the highest value. Intuitively, we are trading off less valuable indices for more valuable indices while taking care to respect the simplex constraint at each step. Our algorithm has two stopping conditions. In the first case, $\widetilde{y}_{s_n}$ attains $y_{s_n} + \epsilon / 2$. This means that we added $\epsilon / 2$ probability mass to $\widetilde{y}_{s_n}$. Hence, we removed $\epsilon / 2$ probability mass from other indices of $\widetilde{y}$, so $\left| \widetilde{y} - y\right|_1 = \epsilon$, i.e. we hit the budget constraint. In the second case, we have $\widetilde{y}_{s_n} = 1$, i.e. we hit the simplex constraint. In both cases, we have moved as much mass as possible from the least valuable indices into the most valuable index.

\begin{theorem}
Given a gradient redirection problem $(G, z, y, \epsilon)$ as formulated in (\ref{equation:gradient_redirection}),\, Algorithm \ref{algorithm:gradient_redirection_l1} outputs a globally optimal solution in $\mathcal{O}\!\left( n \log(n) \right)$ time.
\end{theorem}

The proof follows the common practice for greedy algorithms of establishing the greedy choice property and optimal substructure for a hierarchy of subproblems. The theorem then follows by induction. Please see the Supplementary Material for a full proof.

\noindent\textbf{Computing $Gz$ With Double Backprop.}\quad
Algorithm \ref{algorithm:gradient_redirection_l1} is fast for individual distillation examples $(x, y)$. However, it assumes that we are given $G = \nabla_\theta \log f(x; \theta)$, which itself requires $n$ backward passes through $f$ to compute. In model stealing defenses, $f$ represents a neural network, potentially with millions of parameters. Hence, directly computing $G$ is impractical. To obtain an optimal $\widetilde{y}$ when starting from the raw $(x, y)$ pair, we need a way to circumvent this computational bottleneck.

We solve this problem with double backpropagation. Note that $G$ is only used in Algorithm \ref{algorithm:gradient_redirection_l1} to compute the matrix-vector product $Gz$, which is the gradient with respect to $\widetilde{y}$ of $\langle G^\intercal \widetilde{y},  z\rangle$. We know $G^\intercal \widetilde{y}$ is the gradient of a cross-entropy loss with respect to $\theta$, since we have $G^\intercal \widetilde{y} = \sum_i \widetilde{y}_i \nabla_\theta \log f(x; \theta)_i = - \nabla_\theta H(\widetilde{y}, f(x; \theta))$. Thus, $G^\intercal \widetilde{y}$ can be computed with a single backward pass, and $Gz$ can be computed as $- \nabla_{\widetilde{y}} z^\intercal \nabla_\theta H(\widetilde{y}, f(x; \theta))$ by backpropagating through the computation graph representing the first backward pass, a procedure known as double backpropagation that is supported in many machine learning frameworks. We denote the composition of our gradient redirection algorithm with double backprop as $\text{GR}(f, x, y, z, \epsilon)$ for network $f$, query $x$, clean posterior $y$, target direction $z$, and perturbation budget $\epsilon$. The cost of computing $Gz$ with double backpropagation is roughly comparable to three additional forward passes on top of computing the output $f(x; \theta)$ (one for the first backward pass $+$ two for double backprop), rather than $n$ additional forward passes, as would be necessary when directly computing $G$.


\begin{table*}[t]\centering
\begin{tabularx}{\textwidth}{l *{18}{Y}}\toprule
&\multicolumn{6}{c}{ImageNet-C10 $\to$ CIFAR-10} &\multicolumn{6}{c}{ImageNet-C100 $\to$ CIFAR-100} &\multicolumn{6}{c}{ImageNet-CUB200 $\to$ CUB200} \\\cmidrule{2-19}
&\multicolumn{3}{c}{$\Delta$ Clf. Err} &\multicolumn{3}{c}{$\ell_1$ Distance} &\multicolumn{3}{c}{$\Delta$ Clf. Err} &\multicolumn{3}{c}{$\ell_1$ Distance} &\multicolumn{3}{c}{$\Delta$ Clf. Err} &\multicolumn{3}{c}{$\ell_1$ Distance} \\
Method &1 &2 &5 &0.1 &0.2 &0.5 &1 &2 &5 &0.1 &0.2 &0.5 &1 &2 &5 &0.1 &0.2 &0.5 \\\midrule
Random &9.8 &10.3 &10.6 &9.0 &8.7 &9.3 &38.5 &38.6 &39.8 &36.2 &36.5 &38.5 &48.5 &51.4 &56.0 &41.3 &42.3 &50.7 \\
Reverse Sigmoid &- &- &- &\underline{9.0} &9.1 &9.3 &- &- &- &36.3 &36.8 &38.0 &- &- &- &41.2 &42.6 &45.9 \\
Adaptive Mis. &10.4 &11.9 &16.3 &9.0 &\underline{9.6} &\underline{12.1} &38.2 &40.6 &46.6 &\underline{36.4} &\underline{37.4} &41.8 &\underline{53.8} &\textbf{58.6} &\textbf{66.8} &\textbf{42.8} &\textbf{45.6} &\underline{53.8} \\
MAD &\underline{12.6} &\underline{16.4} &\underline{22.6} &8.7 &8.7 &9.5 &43.0 &46.8 &49.2 &35.9 &36.9 &\underline{42.6} &49.6 &52.3 &56.0 &41.7 &42.6 &51.7 \\
$\text{GRAD}^2$ (Ours) &\textbf{16.4} &\textbf{21.5} &\textbf{23.4} &\textbf{9.5} &\textbf{10.1} &\textbf{15.5} &\textbf{43.4} &\textbf{47.6} &\textbf{53.0} &\textbf{36.5} &\textbf{37.7} &\textbf{44.1} &\textbf{54.1} &\underline{56.4} &\underline{60.7} &\underline{41.8} &\underline{44.6} &\textbf{55.6} \\
\bottomrule
\end{tabularx}
\caption{For distribution-aware adversaries, $\text{GRAD}^2$ increases the adversary's classification error more than state-of-the-art baselines for nearly all budgets and test conditions. All values aside from $\ell_1$ budgets are percentages. For fair comparison, dashes indicate cases where the $\ell_1$ Distance metric is untenable; see text. \textbf{Bold} is best and \underline{underline} is second-best.}\label{tab:distribution-aware}
\vspace{-5pt}
\end{table*}

\subsection{Model Stealing Defense}
We present a new model stealing defense called the Gradient Redirection Adversarial Distillation Defense, abbreviated GRAD${}^2$. Our defense is based on the gradient redirection algorithm. However, since the attacker's network is unknown to the defender, we cannot directly apply the algorithm. Thus, we represent the adversary with a surrogate network $h$ and design optimal perturbations for $h$ instead.

\noindent\textbf{Improved Surrogate Networks.}\quad
Although prior works have used surrogate networks for model stealing defenses \cite{Orekondy2020PredictionPT, wallace2020imitation}, very little is known about how to design good surrogates. In particular, an important unanswered question is whether perturbations designed for the surrogate actually transfer to the attacker's network in the first place. To fill this gap in understanding, we conduct a detailed analysis of various design choices for surrogate networks, reporting results in the Appendix. Our main findings are 1) Perturbations designed for surrogate networks do in fact transfer to the attacker's network, 2) We can train the surrogate on the attacker's queries to obtain better transfer and a stronger downstream defense, and 3) Early stopping of surrogate training leads to a stronger downstream defense. The second finding is a crucial difference between our surrogates and those in previous work. Namely, we notice that the only information we have about the adversary is the query set $\mathcal{Q}$. Thus, to make the surrogate more representative of the adversary, we can train the surrogate on this data using knowledge distillation from the defender's network. In online scenarios, this requires continual learning. For simplicity, we assume the adversary sends all queries in a batch before beginning training.

\noindent\textbf{Coordinated Defense.}\quad
A key advantage of gradient redirection is the freedom of choosing a target direction $z$. We investigate several choices of $z$ and discuss their properties. An intuitive choice is $z = \nabla_\theta H(y, h(x; \theta_h))$, where the target points opposite to the clean gradient. This setting of $z$ is similar to the MAD algorithm from \cite{Orekondy2020PredictionPT}, which finds $\widetilde{y}$ obtaining a high cosine distance with the clean gradient $1 - \cos\left( \widetilde{y}^\intercal G, y^\intercal G \right)$. In Figure \ref{fig:surrogate_metrics}, we compare this setting of $z$ to the MAD algorithm, finding that even though gradient redirection optimizes the inner product as opposed to cosine distance, we outperform MAD on both objectives. When incorporated into a full defense, perturbations from these defenses on two different examples may point in opposite directions and cancel out. Hence, we consider the possibility of coordinating the defense so that perturbations combine constructively rather than destructively.

The simplest possible coordinated defense uses $z = \mathbf{1}$, which pushes the attacker's parameters in the unhelpful direction of the all-ones vector. Since $z$ is constant and does not depend on $x$, it has the desirable property of invariance to batching. That is, computing $\text{GR}(h, x, y, \mathbf{1}, \epsilon)$ in parallel on a batch of inputs gives the same batch of output posteriors whether one obtains $Gz$ via per-example gradients or the batch gradient. In Figure \ref{fig:coordination}, we show that coordinated defenses outperform uncoordinated ones.

\begin{table}[t]\centering
\vspace{5pt}
\begin{tabularx}{0.45\textwidth}{lYYY}\toprule
& &\multicolumn{2}{c}{Attacker Accuracy} \\\cmidrule{3-4}
Eval Data &Defender Accuracy &Knowledge-Limited &Distribution-Aware \\\midrule
CIFAR-10 &94.0 &89.8 &91.0 \\
CIFAR-100 &75.0 &55.5 &64.5 \\
CUB200 &81.4 &58.7 &59.7 \\
\bottomrule
\end{tabularx}
\caption{Accuracy of the defender's classifier and stolen classifiers with no defense applied. Distribution-aware attacks are far more effective than knowledge-limited attacks and hence are more important to defend against.}\label{tab:no_defense}
\vspace{-20pt}
\end{table}

\noindent\textbf{Full $\text{GRAD}^2$ Method.}\quad
Our full gradient redirection defense combines Algorithm \ref{algorithm:gradient_redirection_l1} with our improved surrogates and coordinated defense strategy. Namely, we use a surrogate $h$ trained on the adversary's queries $\mathcal{Q}$ with early stopping after $E=10$ epochs, and we set $z = \mathbf{1}$.

\begin{table*}[t]\centering
\begin{tabularx}{\textwidth}{l *{18}{Y}}\toprule
&\multicolumn{6}{c}{CIFAR-100 $\to$ CIFAR-10} &\multicolumn{6}{c}{CIFAR-10 $\to$ CIFAR-100} &\multicolumn{6}{c}{Caltech256 $\to$ CUB200} \\\cmidrule{2-19}
&\multicolumn{3}{c}{$\Delta$ Clf. Err} &\multicolumn{3}{c}{$\ell_1$ Distance} &\multicolumn{3}{c}{$\Delta$ Clf. Err} &\multicolumn{3}{c}{$\ell_1$ Distance} &\multicolumn{3}{c}{$\Delta$ Clf. Err} &\multicolumn{3}{c}{$\ell_1$ Distance} \\
Method &1 &2 &5 &0.1 &0.2 &0.5 &1 &2 &5 &0.1 &0.2 &0.5 &1 &2 &5 &0.1 &0.2 &0.5 \\\midrule
Random &12.2 &12.8 &14.0 &10.4 &10.5 &11.1 &50.9 &52.1 &54.5 &\underline{46.5} &47.8 &50.6 &53.8 &58.1 &65.1 &43.1 &45.2 &57.1 \\
Reverse Sigmoid &- &- &- &\underline{10.7} &10.5 &11.1 &- &- &- &46.0 &46.9 &50.8 &- &- &- &42.7 &44.2 &49.7 \\
Adaptive Mis. &12.7 &14.3 &21.7 &\textbf{10.8} &\underline{11.5} &\underline{12.9} &47.6 &51.0 &\underline{60.2} &\textbf{47.5} &\textbf{50.6} &\textbf{61.2} &\textbf{64.7} &\textbf{70.6} &- &\underline{43.3} &45.6 &53.4 \\
MAD &\underline{15.1} &\underline{18.0} &\textbf{25.9} &10.5 &10.4 &11.5 &\underline{52.2} &\underline{53.6} &58.6 &45.1 &46.7 &52.0 &55.4 &57.7 &62.1 &\textbf{43.4} &\textbf{47.6} &57.1 \\
$\text{GRAD}^2$ (Ours) &\textbf{17.5} &\textbf{20.5} &\underline{22.4} &10.6 &\textbf{11.7} &\textbf{17.0} &\textbf{55.2} &\textbf{59.3} &\textbf{63.7} &46.3 &\underline{48.0} &\underline{56.8} &\underline{57.9} &\underline{60.7} &\textbf{65.2} &42.5 &\underline{46.1} &\textbf{58.3} \\
\bottomrule
\end{tabularx}
\caption{For knowledge-limited adversaries, $\text{GRAD}^2$ increases the adversary's classification error more than state-of-the-art baselines in most test conditions. All values aside from $\ell_1$ budgets are percentages. For fair comparison, dashes indicate cases where the $\ell_1$ Distance metric is untenable; see text. \textbf{Bold} is best and \underline{underline} is second-best.}\label{tab:knowledge-limited}
\vspace{-10pt}
\end{table*}

\section{Experiments}\label{sec:experiments}
\noindent\textbf{Datasets.}\quad
We use three evaluation datasets: CIFAR-10, CIFAR-100, and CUB200 \cite{krizhevsky2009learning, WelinderEtal2010}. For each evaluation dataset, we explore knowledge-limited and distribution-aware adversaries. The knowledge-limited query sets for the above evaluation datasets are CIFAR-100, CIFAR-10, and Caltech-256 respectively \cite{griffin2007caltech}. For distribution-aware adversaries, we construct query sets from ImageNet-1K by manually selecting overlapping classes \cite{deng2009imagenet}. This gives us ImageNet-C10, ImageNet-C100, and ImageNet-CUB200, which are paired with their matching evaluation set and contain $183,763$, $161,653$, and $30,000$ examples respectively.

\noindent\textbf{Training.}\quad
Our experiments have three stages. In the first stage, a defender network trains on each evaluation dataset. In the second stage, defense methods generate protected posteriors for each query set and evaluation dataset at various defense budgets. Finally, adversary networks train on the protected posteriors. We denote experiments with transfer data $\mathcal{Q}$ and evaluation data $\mathcal{D}$ as $\mathcal{Q} \rightarrow \mathcal{D}$. On CUB200, we fine-tune ResNet50 networks pre-trained on ImageNet. For other datasets, we train $40$-$2$ Wide ResNets from scratch. This allows us to gauge whether pre-training can substantially alter results. All networks are trained for $50$ epochs using SGD with Nesterov momentum of $0.9$. We use an initial learning rate of $0.01$ for CUB200 and $0.1$ for other evaluation datasets. The learning is annealed with a cosine schedule, and we use weight decay of $5 \cdot 10^{-4}$.

\noindent\textbf{Metrics.}\quad
We evaluate defenses with the adversary's classification error on the defender's test set for a given budget. Following \citet{Orekondy2020PredictionPT}, we use two budget metrics: defender classification error and $\ell_1$ distance between the modified posteriors and the clean posteriors averaged across the query set. For metrics in the tables, we report the increase in classification error. We denote these by ``$\Delta$ Clf. Err'' and ``$\ell_1$ Distance'' respectively. A high value for either budget metric renders the defender's network unusable, so strong defenses will induce large adversary errors for small values of both budget metrics. Thus, we focus our comparisons on values of $\Delta$ Clf. Err and $\ell_1$ Distance in a realistically acceptable range of trade-offs.

\begin{figure}[t]
\vspace{-5pt}
\begin{center}
\includegraphics[width=0.44\textwidth]{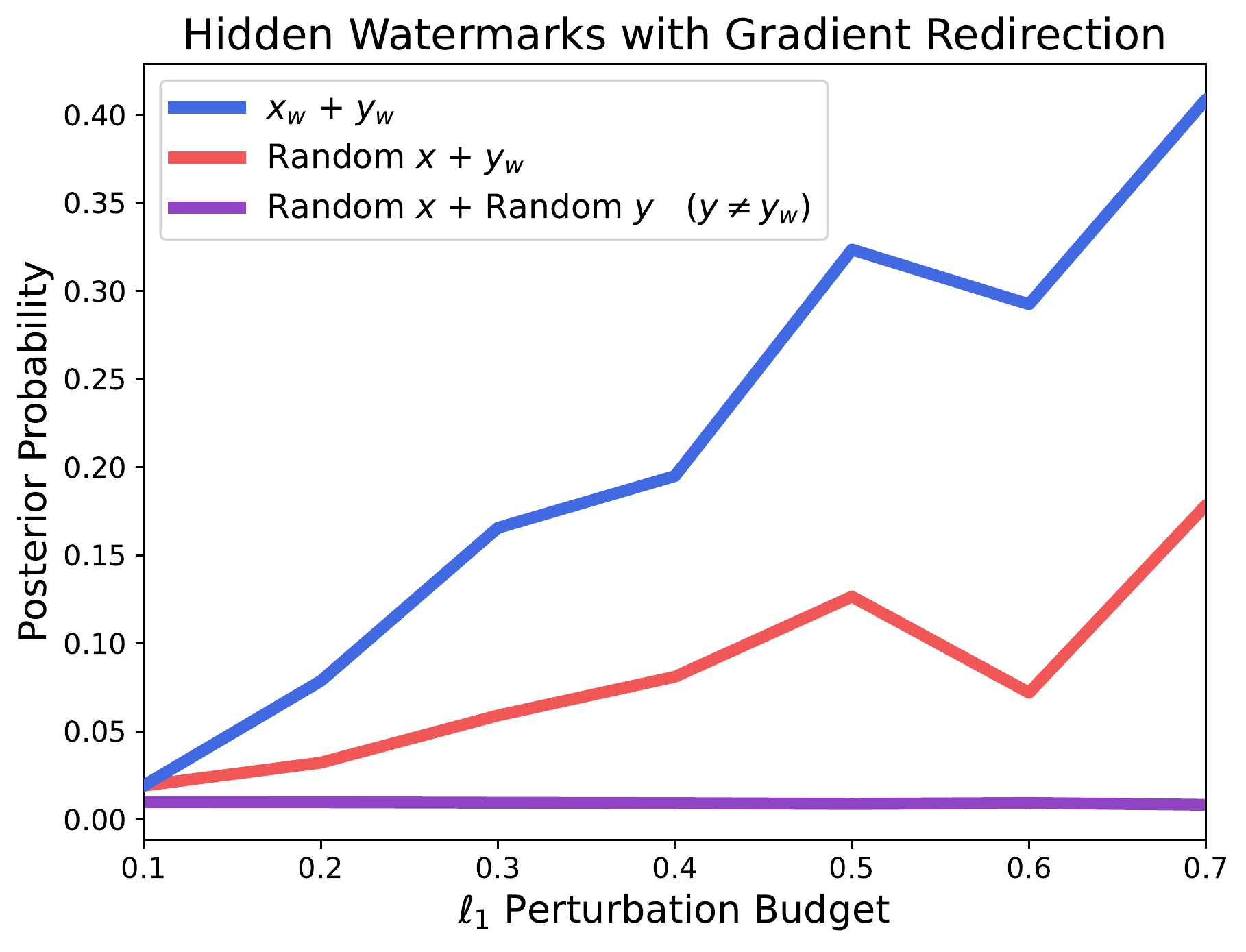}
\end{center}
\vspace{-15pt}
\caption{Gradient redirection enables reprogramming the adversary to predict a desired class on a hidden watermark image that only the defender knows. Posterior probabilities on non-target images and labels is lower than for the target input-output pair.}\label{fig:watermark}
\vspace{-15pt}
\end{figure}

\noindent\textbf{Baselines.}\quad
We compare $\text{GRAD}^2$ to several baseline defenses. \textit{No Defense:} The adversary trains on clean posteriors returned from the API. No attempt is made at defense. \textit{Random:} The defender's posterior is interpolated with a $1$-hot posterior, where the index of the nonzero entry is selected at random from all the labels that are different from the argmax prediction of the clean posterior. \textit{Reverse Sigmoid} \cite{8844598}: The defender's posterior is modified with a noninjective analogue of the sigmoid function, making it challenging for the adversary to exactly recover the clean posterior. \textit{Adaptive Misinformation (AM)} \cite{kariyappa2020defending}: An out-of-distribution (OOD) detector flags anomalous queries as adversarial and interpolates with posteriors from a misinformation network that misclassifies the defender's training set. Note that this method assumes that attacker queries will be OOD, which is a fundamentally different approach from ours. However, we can still compare performance. \textit{MAD} \cite{Orekondy2020PredictionPT}: A heuristic algorithm is used to find perturbed posteriors where the resulting perturbed gradient on a surrogate network has high angular deviation with the clean gradient. The surrogate network is randomly initialized.

\subsection{Comparing Defenses}
We compare defenses at practical values of $\Delta$ Clf Err and $\ell_1$ Distance and show the results in Tables \ref{tab:distribution-aware} and \ref{tab:knowledge-limited}. In Figure \ref{fig:plots1}, we visualize results. For a fairer comparison, tabular results are dashed out if $\ell_1$ Distance is greater than $1.0$, a conservatively large value. This is because the $\Delta$ Clf Err metric can be gamed by simply increasing the temperature of the posterior, which preserves classification error but destroys information. When considering both budget metrics, $\text{GRAD}^2$ matches or outperforms the best prior methods in most cases. Additionally, in terms of raw numbers $\text{GRAD}^2$ often outperforms other defenses by a significant margin. For example, at a classification error budget of $1\%$, our defense increases the error of an adversary seeking to steal a CIFAR-10 model with ImageNet-C10 queries from $9\%$ to $16.4\%$, a $30\%$ relative improvement over the next best method.\looseness=-1

\noindent\textbf{Balancing Both Budget Metrics.}\quad
By examining both the classification error and $\ell_1$ budgets of the defender, we see that methods which perform exceptionally well on one metric can perform very poorly on the other. Namely, Reverse Sigmoid is accuracy-preserving for a large range of hyperparameters, resulting in very high adversary error for minuscule increases in classification error budget. However, its $\ell_1$ Distance budget scales more quickly than all other baselines, indicating that it would not be very useful in practice. By contrast, $\text{GRAD}^2$ has balanced performance on both budget metrics.

\noindent\textbf{Robustness to Threat Model.}\quad
We find that varying the threat model assumptions can significantly affect some defenses. In particular, Adaptive Misinformation relies on being able to identify queries from adversaries with out-of-distribution detectors, so it may perform less well if queries are closer to the defender's training distribution. Accordingly, the gap between $\text{GRAD}^2$ and AM is larger for the distribution-aware adversaries than for knowledge-limited adversaries. With a distribution-aware adversary, $\text{GRAD}^2$ is still successful with a small budget, demonstrating robustness to variations in the threat model.

\subsection{Reprogramming The Adversary}
As we have full control over the target direction $z$ in Algorithm \ref{algorithm:gradient_redirection_l1}, a natural question is whether gradient redirection can be used for more than just increasing the adversary's error. We answer this in the affirmative by showing that in ideal conditions gradient redirection can reprogram adversaries to behave in a desired way on watermark images. Importantly, unlike prior work on watermarking defenses for prediction APIs \cite{szyller2021dawn}, we can select our watermark images at will. We are not restricted to using images that the adversary has sent as queries.

Let $(x_w, y_w)$ be an input-output pair that we want the adversary $f$ to predict. To demonstrate watermarking, we assume ideal conditions with white-box access to $f$. Let the gradient redirection target be $z = -\nabla_\theta H(y_w, f(x_w; \theta))$. This target is updated after every training step. We experiment with knowledge-limited adversaries on the CIFAR datasets. For each experiment, we insert a single $(x_w, y_w)$ watermark into $f$. We perform each experiment three times with different randomly-selected $(x_w, y_w)$ pairs, where $x_w$ is selected from the defender's test set, and we average results. In Figure \ref{fig:watermark}, we plot the posterior probability $f(x_w)_{y_w}$ at various perturbation budgets for the converged stolen model $f$. We also plot the average value of $f(x)_{y_w}$ when given random inputs $x$ from the test set and $f(x)_y$ for random $x$ and $y$. Gradient redirection enables reprogramming the adversary to have abnormally high posteriors on the watermark input-output pair.

\begin{figure}[t]
\vspace{5pt}
\begin{center}
\includegraphics[width=0.46\textwidth]{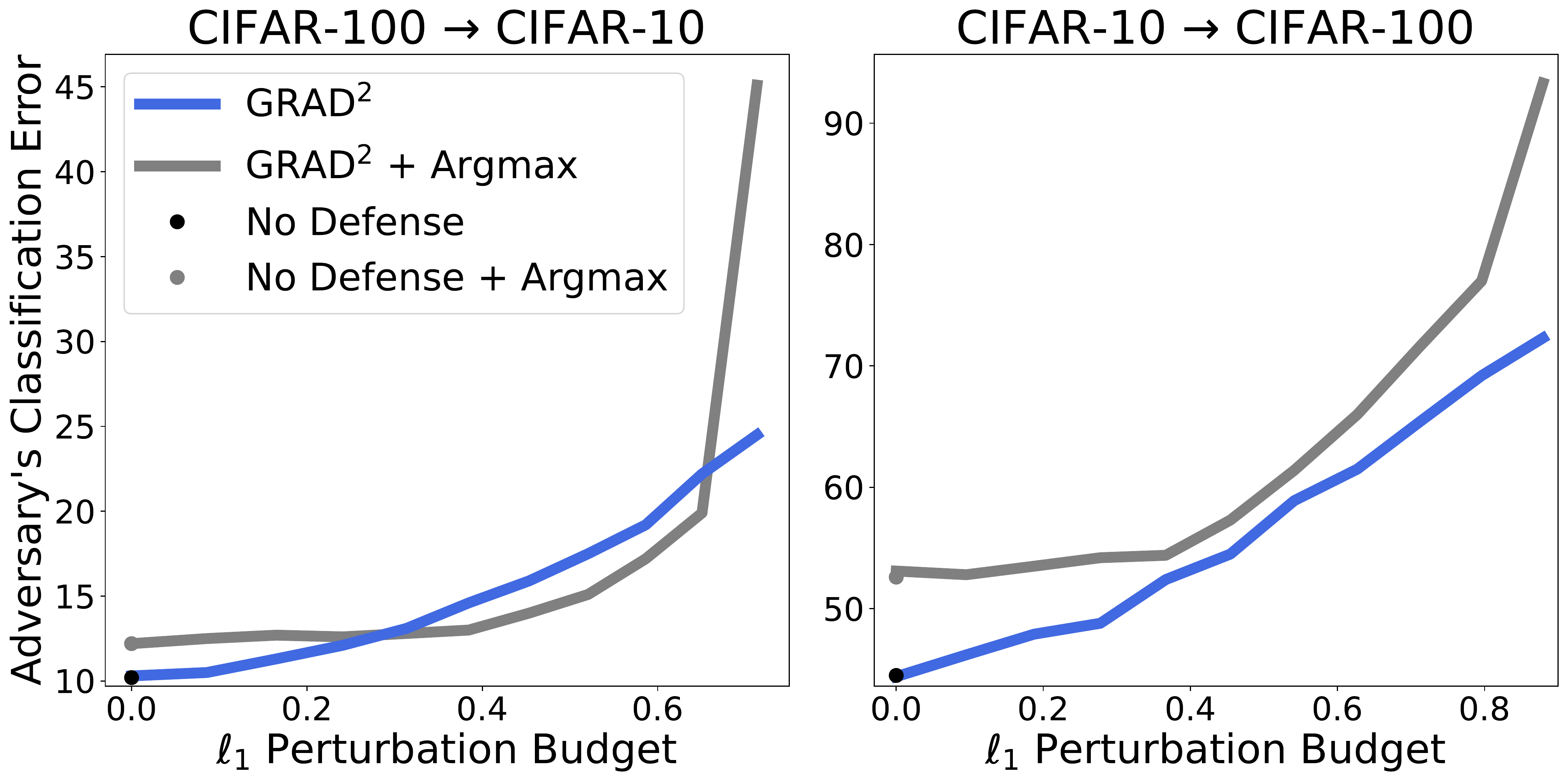}
\end{center}
\vspace{-5pt}
\caption{$\text{GRAD}^2$ remains a strong defense even when the adversary employs countermeasures. Training on the argmax label alone results in a higher error for the adversary in most settings, demonstrating the robustness of $\text{GRAD}^2$.}\label{fig:argmax}
\vspace{-10pt}
\end{figure}

\subsection{Adversary Countermeasures}
To evaluate how adversary countermeasures affect our $\text{GRAD}^2$ defense, we train knowledge-limited adversaries on CIFAR-10 and CIFAR-100 with the argmax label of the perturbed posteriors rather than the entire posterior. In Figure \ref{fig:argmax}, we plot results. We find that adversaries trained with this countermeasure obtain a higher initial error, although they can preserve accuracy in the face of more aggressive defenses. $\text{GRAD}^2$ remains a strong defense in the face of this countermeasure, with a higher adversary error in most cases, demonstrating the robustness of our approach.

\section{Conclusion}
We introduced gradient redirection, a new approach to model stealing defenses that enables modifying the adversary's update gradient in a targeted manner. We presented a provably optimal algorithm to efficiently solve gradient redirection problems, which we use to construct the $\text{GRAD}^2$ model stealing defense. In experiments, our defense outperformed all prior defenses and was robust to adversary countermeasures. Moreover, we showed that gradient redirection can be used to reprogram the adversary in a desired manner, which we hope will help foster further work on model stealing defenses.

\section*{Acknowledgements}
This work is partially supported by
NSF grant No.1910100,
NSF No.2046726, C3 AI, and the  Alfred P. Sloan Foundation.

\bibliography{main}
\bibliographystyle{icml2022}

\appendix
\section{Proving Optimality}
\subsection{Gradient Redirection Problem}
For simplicity, we start by assuming white-box access to the adversary's network $f$. We operationalize the defender's goal as minimally perturbing the defender's posterior $y$ in order to maximally push the adversary's update gradient in a target direction $z \in \mathbb{R}^d$. For a given distillation example $(x, y)$, we want to solve the optimization problem
\begin{align}\label{equation:gradient_redirection}
    \max_{\widetilde{y}} \quad & \langle G^\intercal \widetilde{y},  z\rangle \\
    \textrm{s.t.} \quad & \mathbf{1}^\intercal \widetilde{y} = 1\nonumber \\
    & \widetilde{y} \succeq 0\nonumber \\
    & \| \widetilde{y} - y \|_1 \leq \epsilon,\nonumber
\end{align}
where $y$, $G = \nabla_\theta \log f(x; \theta)$, $z \in \mathbb{R}^d$, and $0 \leq \epsilon < 2$ are fixed. Typically we also have $y \in \Delta^{n-1}$, although for establishing optimal substructure we may have that $\sum_i y_i < 1$. Note that this is a linear program reminiscent of knapsack problems.

\subsection{Gradient Redirection Algorithm}

We propose an efficient and provably optimal algorithm to solve the gradient redirection problem, which we describe in the main paper. First, note that the inner product objective (\ref{equation:gradient_redirection}) can be rewritten as $\widetilde{y}^\intercal G z$, where $Gz \in \mathbb{R}^n$ can be interpreted as a value vector. For brevity, let $c = Gz$, and let $s = \argsort(c)$, i.e. $c_{s_1} \leq c_{s_2} \leq \dotsb \leq c_{s_n}$.

Our algorithm initializes $\widetilde{y}$ as $y$. It then proceeds by taking probability mass from indices of $c$ with low values and putting as much mass as possible in $c_{s_n}$, which has the highest value of all dimensions in $c$. Intuitively, we are trading off less valuable indices for more valuable indices while taking care to respect the simplex constraint at each step. Our algorithm has two stopping conditions. In the first case, $\widetilde{y}_{s_n}$ attains $y_{s_n} + \epsilon / 2$. This means that we added $\epsilon / 2$ probability mass to $\widetilde{y}_{s_n}$. Hence, we removed $\epsilon / 2$ probability mass from other indices of $\widetilde{y}$, so $\left| \widetilde{y} - y\right|_1 = \epsilon$, i.e. we hit the budget constraint. In the second case, we have $\widetilde{y}_{s_n} = 1$, i.e. we hit the simplex constraint. In both cases, we have moved as much mass as possible from the least valuable indices into the most valuable index.

\begin{lemma}[Greedy Choice Property]\label{lemma:greedy_choice}
Let $(G, z, y, \epsilon)$ be a gradient redirection problem as formulated in (\ref{equation:gradient_redirection}),\, but with the budget constraint changed to $\left| \widetilde{y} - y \right|_1 + \lambda \leq \epsilon$, where $\lambda = 1 - \sum_i y_i$ ($\lambda = 0$ in the original problem). Let $y^*$ be an optimal solution. Then we must have (a): $y^*_{s_n} = \min\!\left( y_{s_n} + \epsilon / 2,\, 1 \right)$ and (b): $y^*_{s_1} = \max\!\left( y_{s_1} - (\epsilon / 2 - \lambda),\, 0 \right)$. Furthermore, if $y^*_{s_1} \neq 0$, then we have (c): $y^*_{s_t} = y_{s_t}$ for $1 < t < n$.
\end{lemma}
\begin{proof}
First we will prove (a). Assume for contradiction that there is an optimal solution $y^\circ$ such that $y^\circ_{s_n} \neq \min\!\left( y_{s_n} + \epsilon / 2,\, 1 \right)$. Consider the first case, where $1 \leq y_{s_n} + \epsilon / 2$. This implies $y^\circ_{s_n} < 1$. Intuitively, we can find a feasible solution $v$ with $v_{s_n} = 1$ that obtains a higher objective value than $y^\circ$. Let $v$ be the posterior with $v_{s_n} = 1$ and $v_i = 0$ for $i \neq s_n$. In the present case, we have $1 - y_{s_n} \leq \epsilon / 2$, which implies $\left| v - y \right|_1 + \lambda = (1 - y_{s_n}) + \sum_{1 \leq t < n} y_{s_t} + (1 - \sum_i y_i) = 2(1 - y_{s_n}) \leq \epsilon$. Thus, $v$ is a feasible solution. Since $y^\circ$ and $v$ both sum to $1$, we know $v$ has a higher objective value than $y^\circ$, which is a contradiction.

Now consider the second case, where $1 > y_{s_n} + \epsilon / 2$. This implies $y^\circ_{s_n} < y_{s_n} + \epsilon / 2 < 1$. Intuitively, we can find a feasible solution $v$ that obtains a higher objective value than $y^\circ$ by moving mass into index $s_n$. Let $I$ be the indices where $y^\circ_i \geq y_i$, and let $J$ be the indices where $y^\circ_j < y_j$. We have $\left| y^\circ - y \right|_1 + \lambda = (\sum_{i \in I} y^\circ_i - y_i) +  (\sum_{j \in J} y_j - y^\circ_j) + 1 - \sum_i y_i = 1 + (\sum_{i \in I} y^\circ_i - 2y_i) -  \sum_{j \in J} y^\circ_j = 2\sum_{i \in I} y^\circ_i - y_i$, so we have $\sum_{i \in I} y^\circ_i - y_i \leq \epsilon / 2$. Suppose we have $\left| y^\circ - y \right|_1 + \lambda < \epsilon$. Then we can just increase $y^\circ_{s_n}$ and decrease other entries of $y^\circ$ to maintain the simplex constraint until we hit the budget constraint, which would yield a solution with a higher objective value. Now suppose we have $\left| y^\circ - y \right|_1 + \lambda = \epsilon$. This implies $\sum_{i \in I} y^\circ_i - y_i = \epsilon / 2$. In other words, the indices where $y^\circ_i > y_i$ account for a difference of $\epsilon / 2$, but they are not concentrated in $y^\circ_{s_n}$ by our assumption. Simply move them to $y^\circ_{s_n}$ to obtain another feasible solution with greater objective value than that of $y^\circ$. This is a contradiction for the second case, so we have proven (a).

Now we will prove (b). In the case where $1 \leq y_{s_n} + \epsilon / 2$, we know $y^*_{s_n} = 1$ and $y^*_i = 0$ for $i \neq s_n$ from part (a). We also have $y_{s_1} - (\epsilon / 2 - \lambda) < (\sum_{i \neq s_n} y_i) - \epsilon / 2 + 1 - \sum_i y_i = 1 - \epsilon / 2 - y_{s_n} < 0$, so $\max\!\left( y_{s_1} - (\epsilon / 2 - \lambda),\, 0 \right) = 0$, and we know $y^*_{s_1} = 0$. Hence, (b) is true in the case where $1 \leq y_{s_n} + \epsilon / 2$.

Now consider the case where $1 > y_{s_n} + \epsilon / 2$. Assume for contradiction that there is an optimal solution $y^\circ$ such that $y^\circ_{s_1} \neq \max\!\left( y_{s_1} - (\epsilon / 2 - \lambda),\, 0 \right)$. By (a), we know $y^\circ_{s_n} = y_{s_n} + \epsilon / 2$. Let $I$ be the indices where $y^\circ_i > y_i$, and let $J$ be the indices where $y^\circ_j \leq y_j$. By analogous argument to that in part (a), we have $\sum_{i \in I} y^\circ_i - y_i = \epsilon / 2$. That is, indices where $y^\circ_i > y_i$ account for a difference of exactly $\epsilon / 2$. But this means that $s_n$ is the only such index, so we have $y^\circ_j \leq y_j$ for $j \neq s_n$. Furthermore, we have $\left| y^\circ - y \right|_1 + \lambda = (\sum_{i \in I} y^\circ_i - y_i) +  (\sum_{j \in J} y_j - y^\circ_j) + \lambda = \epsilon / 2 + (\sum_{j \in J} y_j - y^\circ_j) + \lambda = \epsilon$, so we have $(\sum_{j \in J} y_j - y^\circ_j) = \epsilon / 2 - \lambda$. This means that the indices $j \neq s_n$ account for a difference of exactly $\epsilon / 2 - \lambda$. By concentrating this difference in $y^\circ_{s_1}$, we can improve the objective value while maintaining a feasible solution. If $y_{s_1} - (\epsilon / 2 - \lambda) \geq 0$, we will be able to concentrate all this difference into $y^\circ_{s_1}$, in which case $y^\circ_{s_t} = y_{s_t}$ for $ 1 < t < n$. If $y_{s_1} - (\epsilon / 2 - \lambda) < 0$, we will not be able to concentrate all of the difference into $y^\circ_{s_1}$ due to the nonnegativity constraint. However, concentrating as much of the difference as possible by setting $y^\circ_{s_1} = 1$ will still give a feasible solution with an improved objective value. This is a contradiction, so we have proven (b).

Now we will prove (c). Suppose $y^*$ is an optimal solution with $y^*_{s_1} \neq 0$. In the proof of part (b), we saw that this only happens when we have $y_{s_1} - (\epsilon / 2 - \lambda) \geq 0$. Furthermore, in this case we also have $y^*_{s_t} = y_{s_t}$ for $ 1 < t < n$, because we had to move all the remaining difference between $y^*$ and $y$ (outside of indices $s_n$ and $s_1$) into decreasing $y^*_{s_1}$ as much as possible. This proves (c).


\end{proof}

\begin{lemma}[Optimal Substructure]\label{lemma:optimal_substructure}
Let $(G, z, y, \epsilon)$ be a gradient redirection problem with the modified budget constraint from Lemma \ref{lemma:greedy_choice}. Namely, $\left| \widetilde{y} - y \right|_1 + \lambda \leq \epsilon$, where $\lambda = 1 - \sum_i y_i$. Let $y^*$ be an optimal solution with $y^*_{s_1} = 0$, and let $s = \argsort(Gz)$. Consider the subproblem $(G', z', y', \epsilon)$, where $G'$ has row $s_1$ removed, $z'$ and $y'$ have index $s_1$ removed. Then $y^*$ with index $s_1$ removed is an optimal solution to the subproblem. (Note that in the subproblem, $y$ may lie outside the simplex, but the simplex constraint for $\widetilde{y}$ is unchanged.)
\end{lemma}
\begin{proof}
Let $y^*{}'$ denote $y^*$ with index $s_1$ removed. We know $\lambda' = 1 - \sum_i y'_i = \lambda + y_{s_1}$, and we know $\sum_i y_i^*{}' = \sum_i y_i^* = 1$. Thus, we have $\left| y^*{}' - y' \right|_1 + \lambda' = \left| y^* - y \right|_1 - \left| y^*_{s_1} - y_{s_1} \right| + \lambda + y_{s_1} = \left| y^* - y \right|_1 + \lambda \leq \epsilon$. Therefore, $y^*{}'$ is a feasible solution to the subproblem $(G', z', y', \epsilon)$. Assume for contradiction that there is a solution $y^\circ$ to the subproblem with a higher objective value than $y^*{}'$. Insert $0$ into position $s_1$ of $y^\circ$ to create a solution to the original problem with the same objective value. Inserting $0$ into position $s_1$ of $y^*{}'$ also creates a solution to the original problem with the same objective value. In fact, it recreates $y^*$. But $y^*$ is an optimal solution to the original problem, so its objective value cannot be lower than that of the expanded $y^\circ$. This is a contradiction, so our assumption must be false, which means that $y^*{}'$ is optimal for the subproblem.
\end{proof}

\begin{theorem}
Given a gradient redirection problem $(G, z, y, \epsilon)$ as formulated in (\ref{equation:gradient_redirection}),\, Algorithm \ref{algorithm:gradient_redirection_l1} outputs a globally optimal solution in $\mathcal{O}\!\left( n \log(n) \right)$ time.
\end{theorem}

\begin{proof}
By part (a) of Lemma \ref{lemma:greedy_choice}, we know that the initialization steps of Algorithm \ref{algorithm:gradient_redirection_l1} set $\widetilde{y}_{s_n}$ to the value of the optimal solution. If $\widetilde{y}_{s_n}$ is set to $1$, then the while loop will proceed until its stopping condition without altering $\widetilde{y}$ at other indices, returning an optimal $\widetilde{y}$. Now consider the case where $\widetilde{y}_{s_n}$ is set to a value less than $1$, then we know $y$ has at least $\epsilon / 2$ mass at indices other than $s_n$. This means $\lambda + y_{s_t}$ will eventually exceed $\epsilon / 2$, at which point the if-then condition will be entered and the algorithm will return.

We proceed by induction, starting where the if-then condition is entered and continuing backwards through the while loop to the first step. Note that when the if-then condition is entered, we have just set $\widetilde{y}_{s_t}$ to $y_{s_t} - (\epsilon / 2 - \lambda)$. Let $(G', z', y', \epsilon)$ be the subproblem with $s_j$ indices removed for $j < t$. In this subproblem, $\widetilde{y}_{s_t}$ is replaced by $\widetilde{y}'_{s_1}$. From part (b) of Lemma \ref{lemma:greedy_choice}, we know $\widetilde{y}'_{s_1}$ is set to an optimal value for this subproblem. Moreover, from part (c) and the early exit step of the algorithm, we know that the entirety of $\widetilde{y}'$ returned at this point is an optimal solution to the subproblem, not just the values at indices $s_1$ and $s_{n-t}$. This gives us a base case. For the induction step, assume that $\widetilde{y}'$ is optimal at step $t > 1$ of the while loop. We want to show that $\widetilde{y}$ is optimal for the subproblem $(G, z, y, \epsilon)$ at step $t - 1$ of the while loop. By part (b) of Lemma \ref{lemma:greedy_choice}, we know that $\widetilde{y}'_{s_1}$ is set to an optimal value in the previous step of the while loop. By Lemma \ref{lemma:optimal_substructure}, we know that simply appending this value to the optimal solution of the subproblem at step $t$ of the while loop gives us an optimal solution to the subproblem at step $t-1$. By induction, the complete $\widetilde{y}$ returned by Algorithm \ref{algorithm:gradient_redirection_l1} is an optimal solution to the original gradient redirection problem.

Asymptotically, the most expensive step is the sorting operation. All other steps are local comparisons and operations on pairs of array elements or scalars, and hence have $\mathcal{O}\!\left( n \right)$ time complexity. Thus, the overall time complexity is $\mathcal{O}\!\left( n \log(n) \right)$.
\end{proof}

\section{Additional Experiments}
In Figures \ref{fig:surrogate_analysis} and \ref{fig:coordination}, we analyze the effect of our design choices for the $\text{GRAD}^2$ defense. Namely, we show how the surrogate ability to transfer to the adversary makes a large difference in downstream performance. In Figure \ref{fig:coordination}, we plot the performance of a coordinated defense and uncoordinated defense. These analyses identify ways to build stronger defenses which future work could build on. Additionally, they highlight the general importance of surrogate networks and destructive interference, which prior work did not investigate in detail.

\subsection{Improved Surrogate Networks}
We analyze the effect of design choices for the surrogate network. Namely, we show how the surrogate network's ability to transfer to the adversary makes a large difference in the downstream performance of gradient redirection defenses. Let $f$ and $h$ be an adversary and surrogate network, respectively. Let $\theta$ be the parameters of $f$ that the adversary trains during a model stealing attack. We compute Transfer Performance for the surrogate network $h$ at perturbation budget $\epsilon$ as
\[\frac{1}{\left| \mathcal{Q} \right|} \sum_{x \in \mathcal{Q}}\cos\left( \widetilde{y}^T \nabla_\theta \log f(x; \theta), z \right) - \cos\left( y^T \nabla_\theta \log f(x; \theta), z \right)\]

where $y$ is the defender model's output for query $x$, $z$ is the gradient redirection target, and $\widetilde{y} = \text{GR}(h, x, y, \mathbf{1}, \epsilon)$ is the output of gradient redirection on the surrogate network. Transfer Performance measures the increase in cosine similarity between the perturbed gradient and the target direction $z$ compared to the negative control of the cosine similarity between the clean gradient and $z$. Thus, Transfer Performance greater than zero indicates successful transfer, and values less than or equal to zero indicate failed transfer. While any gradient redirection target could be used, we focus on the all-ones target $z = \mathbf{1}$ in this analysis.

\noindent\textbf{Setup.}\quad
The surrogate network analysis is in Figure \ref{fig:surrogate_analysis}. We train surrogates using the Knockoff Nets objective with the query distribution $\mathcal{Q}$ (sold lines) and the defender's training distribution $\mathcal{D}$ (dashed lines). In both cases, we train surrogates for $50$ epochs and perform early stopping after $E$ epochs, where $E \in \{0, 10, 20, 30, 40\}$. We then train three independent adversaries for $50$ epochs each, saving snapshots at each epoch. We compute Transfer Performance for a surrogate at each epoch of the adversary's training, averaging over the three independent adversary training runs. We also compute the converged adversary's classification error when using the surrogates in a $\text{GRAD}^2$ defense.

\noindent\textbf{Training Surrogates on the Query Distribution.}\quad
The only concrete information the defender has about an adversary is the query distribution. A natural question is whether we can leverage this information to build a stronger defense. In the top row of Figure \ref{fig:surrogate_analysis}, we see that Transfer Performance is higher for surrogates trained on $\mathcal{Q}$. In the bottom row, we see that this translates to a stronger downstream defense for trained surrogates. This suggests that better transfer from the surrogate to the adversary is in fact valuable for the downstream defense. Note that prior works found that trained surrogates yield weaker defenses than random surrogates \cite{Orekondy2020PredictionPT}, which may be an artifact of training the surrogates on $\mathcal{D}$ instead of $\mathcal{Q}$. We show that training on $\mathcal{Q}$ causes trained surrogates to outperform random surrogates by a significant margin.

\noindent\textbf{Early Stopping of Surrogates.}\quad
In the top row of Figure \ref{fig:surrogate_analysis}, we can see that surrogates that train for longer transfer better to adversaries at later stages of training, and early stopping of surrogates yields better transfer to adversaries at the initial stages of training. This provides further evidence that similarity to the adversary affects transfer and raises the question of whether early transfer or late transfer is more advantageous for a downstream defense. For both CIFAR-10 and CIFAR-100, we find that early stopping at $E=10$ epochs yields the strongest defense in most cases, while values of $E \in \{20, 30, 40\}$ perform similarly. This suggests that early transfer is more important than late transfer.

\subsection{Runtime Comparison}
In Table \ref{tab:runtime}, we compare the runtime of MAD and $\text{GRAD}^2$ on queries from CIFAR-10, CIFAR-100, and CUB200. We report the number of seconds to generate a perturbed posterior averaged across the test set. In all cases, $\text{GRAD}^2$ is substantially faster, and relative performance scales with the number of labels; $\text{GRAD}^2$ is $3.75 \times$ faster on CIFAR-10, $4.32 \times$ faster on CIFAR-100, and $6.33 \times$ faster on CUB200.

\begin{figure}[t]
\vspace{5pt}
\begin{center}
\includegraphics[width=0.46\textwidth]{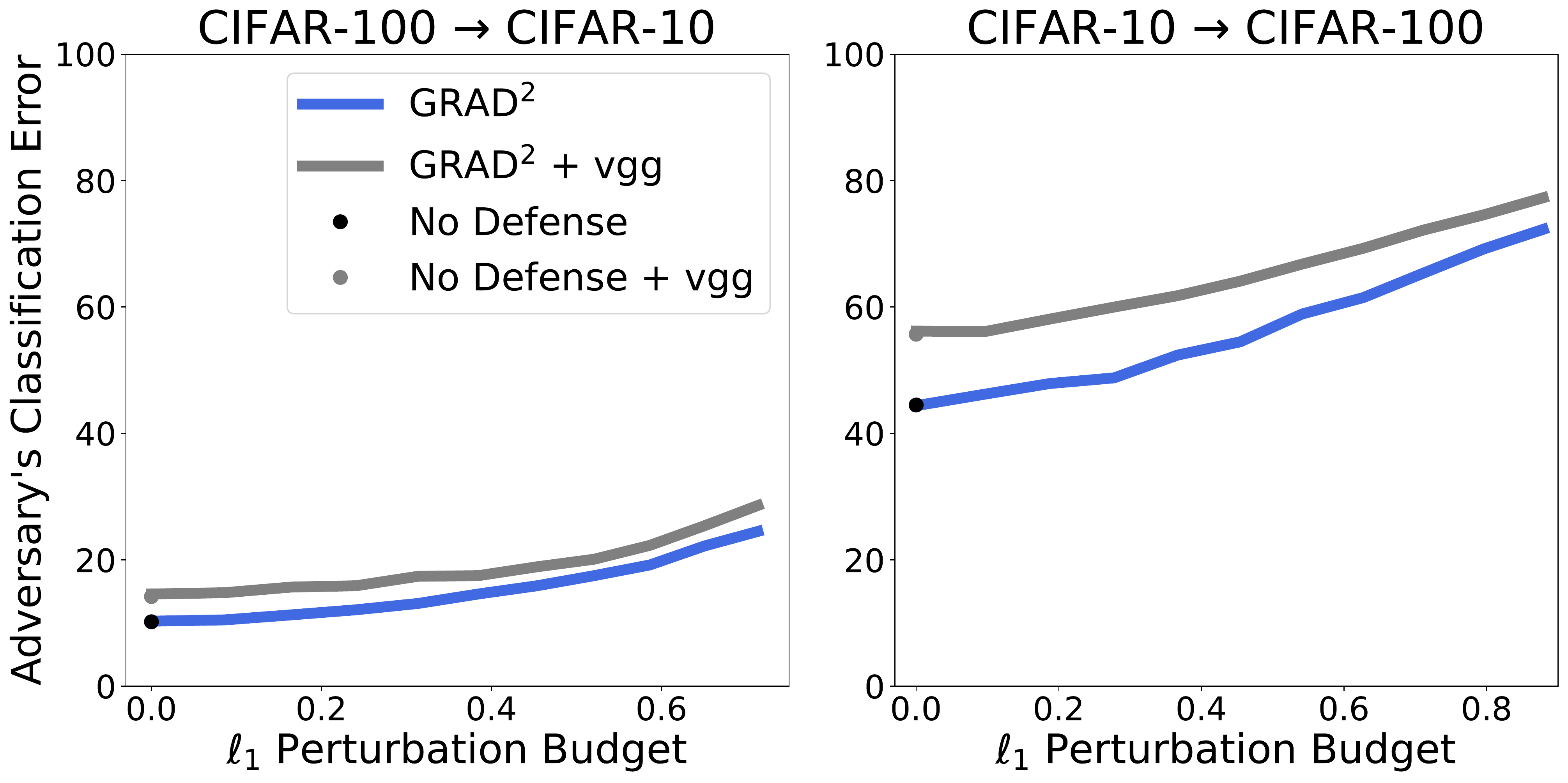}
\end{center}
\vspace{-5pt}
\caption{Posterior perturbations designed for a particular surrogate remain effective when the adversary uses a different network architecture.}\label{fig:architecture_robustness}
\end{figure}

\subsection{Robustness Across Architectures}\label{app:architecture_robustness}
An important practical consideration for using gradient redirection defenses is whether surrogates remain effective if the adversary uses a different neural network architecture than the surrogate. In Figure \ref{fig:architecture_robustness}, we evaluate $\text{GRAD}^2$ against adversaries using a different architecture than the standard surrogate used throughout the paper. In particular, we keep the same WRN-40-2 surrogate and posterior perturbations from the main experiments, but we change the adversary's architecture to VGG-16. We find that $\text{GRAD}^2$ remains a strong defense in this setting, with almost no decrease in slope to the performance profile. This suggests that surrogates do transfer to different architectures.

\begin{table}
\begin{center}
\begin{tabular}{lccc}
\toprule
    Method & CIFAR-10 & CIFAR-100 & CUB200 \\ \midrule
MAD & 0.15 & 1.21 & 2.66 \\
$\text{GRAD}^2$ & 0.04 & 0.28 & 0.42 \\
\bottomrule
\end{tabular}
\end{center}
\caption{Average time in seconds to generate a perturbed posterior for a single query on an NVIDIA A40 GPU. $\text{GRAD}^2$ is significantly faster than MAD.}
\label{tab:runtime}
\end{table}

\begin{figure*}[t]
\begin{center}
\includegraphics[width=\textwidth]{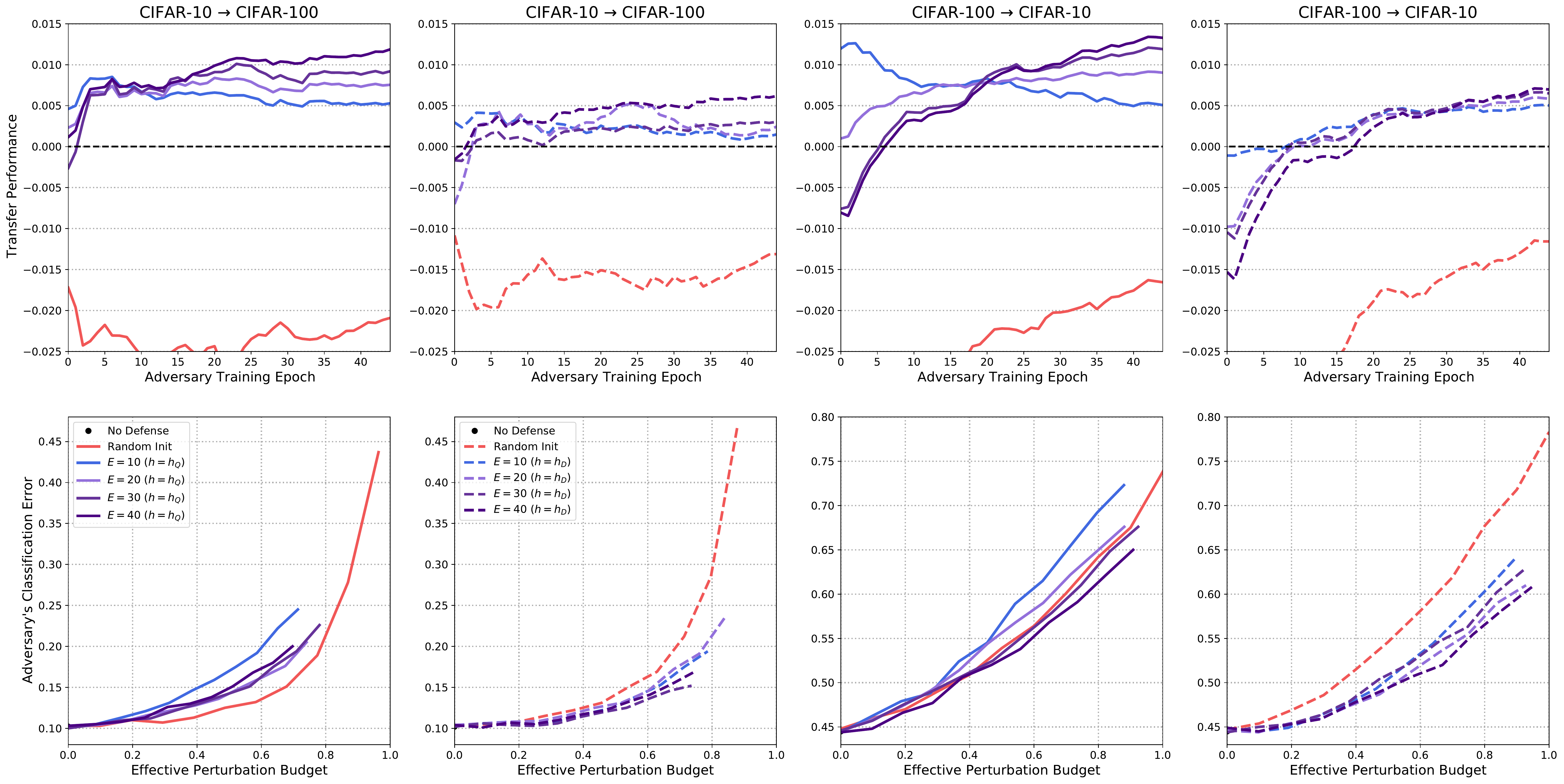}
\end{center}
\vspace{0pt}
\caption{
Surrogate analysis experiments. Top: Surrogates do transfer to adversaries with unknown parameters; Surrogates with early stopping after a small number of epochs transfer better to adversaries early in their training, whereas surrogates with early stopping after a larger number of epochs transfer better to adversaries later in their training; Surrogates trained with the query distribution $\mathcal{Q}$ (solid lines) transfer more effectively than surrogates trained with the defender's training distribution $\mathcal{D}$ (dashed lines); Random surrogates transfer very poorly (red lines). Bottom: Surrogate transfer matters, as stronger transfer results in a stronger downstream defense (higher adversary error for a given $\ell_1$ budget); Surrogates with early stopping at low epochs ($E=10$) yield a stronger downstream defense across datasets (blue lines); Randomly initialized surrogates (red lines) yield in a relatively weaker downstream defense, which surrogates trained on $\mathcal{Q}$ outperform. Notably, the MAD method uses randomly initialized surrogates, which transfer very poorly and underperform our surrogates trained on $\mathcal{Q}$ with early stopping at $E=10$, which are used in the $\text{GRAD}^2$ defense.
}\label{fig:surrogate_analysis}
\vspace{-10pt}
\end{figure*}

\begin{figure}[t]
\begin{center}
\includegraphics[width=0.46\textwidth]{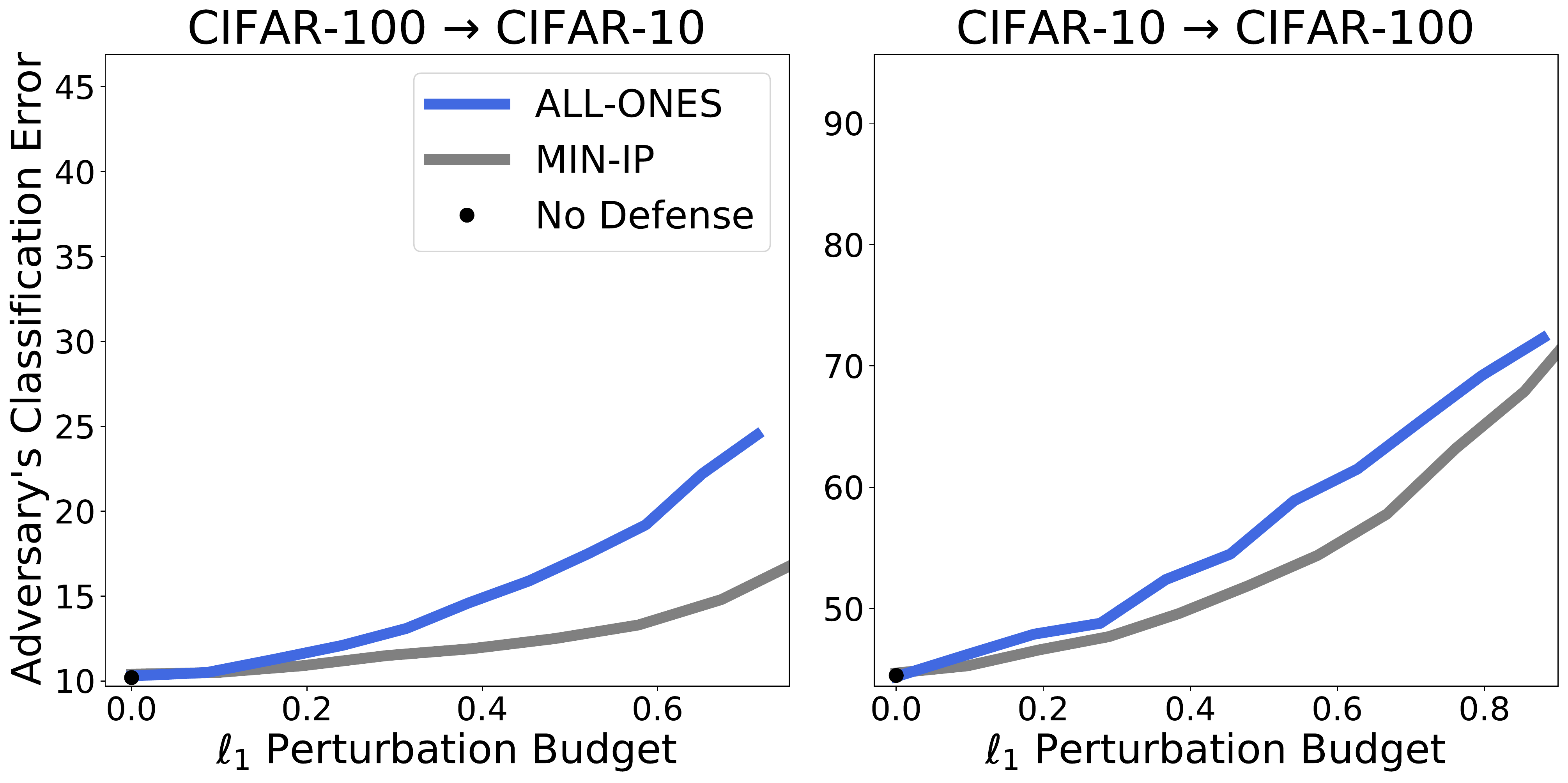}
\end{center}
\vspace{-10pt}
\caption{When we set the gradient redirection target $z$ to the all-ones vector, this results in a defense that perturbs the adversary in a coordinated manner, providing that the perturbations transfer from the surrogate to the adversary. Perturbing away from the clean gradient, i.e. minimizing the inner product, may result in destructive interference of perturbations across the training set. Here, we measure the performance of these two choices of $z$, labeled ALL-ONES and MIN-IP respectively, where MIN-IP corresponds to $z = \nabla_\theta H(y, h(x; \theta_h))$ pointing opposite the clean. ALL-ONES corresponds to $z = \mathbb{1}$, which does not depend on $x$ and thus has a coordinated effect across training batches. We find that our coordinated defense outperforms the uncoordinated defense by a substantial margin.}\label{fig:coordination}
\vspace{-10pt}
\end{figure}








\end{document}